\documentclass[acmsmall,nonacm]{acmart}

\usepackage{tikz}
\usetikzlibrary{automata, positioning, arrows, shapes, decorations.text, fpu}

\tikzset{
    >=stealth',
    punkt/.style={
           rectangle,
           rounded corners,
           draw=black, very thick,
           text width=6.5em,
           minimum height=2em,
           text centered},
    pil/.style={
           ->,
           thick,
           shorten <=2pt,
           shorten >=2pt,},
    box/.style={
           rectangle,
           dashed,
           draw=black,
           very thick,
           inner sep=4mm,
           rounded corners
    }
}

\usepackage{amsmath}
\usepackage{algorithm2e}
\DeclareMathOperator*{\argmax}{arg\,max}
\usepackage{xcolor}
\usepackage{colortbl}
\AtBeginDocument{%
  }

\newtheorem{theorem}{Theorem}

\usepackage{subfig}

\newcommand{\vocab}{\mathcal{M}}
\newcommand{\corpus}{\mathcal{S}}
\newcommand{\traincorpus}{\hat{\mathcal{S}}}
\DeclareMathOperator{\sign}{sign}
\usepackage{comment}
\usepackage{mathrsfs}

\usepackage[colorinlistoftodos]{todonotes}

\begin{document}



\title{Beyond the Black Box: A Statistical Model for LLM Reasoning and Inference}

\author{
    Siddhartha Dalal,
    Vishal Misra
    }

\thanks{Siddhartha Dalal is with the Statistics Department and School of Professional Studies, Columbia University, New York, 10027 NY USA (e-mail: sd2803@columbia.edu)}
\thanks{Vishal Misra is with the Computer Science Department, Columbia University, New York, 10027 NY USA (e-mail: vishal.misra@columbia.edu)}

\begin{abstract}

This paper introduces a novel Bayesian learning model to explain the behavior of Large Language Models (LLMs), focusing on their core optimization metric of next token prediction. We develop a theoretical framework based on an ideal generative text model represented by a multinomial transition probability matrix with a prior, and examine how LLMs approximate this matrix. Key contributions include:

(i) a continuity theorem relating embeddings to multinomial distributions,
(ii) a demonstration that LLM text generation aligns with Bayesian learning principles,
(iii) an explanation for the emergence of in-context learning in larger models,
(iv) empirical validation using visualizations of next token probabilities from an instrumented Llama model

Our findings provide new insights into LLM functioning, offering a statistical foundation for understanding their capabilities and limitations. This framework has implications for LLM design, training, and application, potentially guiding future developments in the field.

\end{abstract}
\maketitle
\section{Introduction}

The advent of Large Language Models (LLMs), starting with GPT3 \cite{brown2020language}, has revolutionized the world of natural language processing, and the introduction of ChatGPT \cite{openai2023chatgpt} has taken the world by storm. There have been several approaches to try and understand how these models work, and in particular how "few-shot" or ``in context learning'' work \cite{liu2022transformers, Lightman2023LetsVerify, li2023symbolic}. In our work we look at the workings of an LLM from a novel standpoint, and develop a Bayesian model to explain their behavior. Bayesian models are natural here since tokens are being generated based on the past training data that supplies prior and the prompts supply observations which updates the prior. We focus on the optimization metric of next token prediction for these LLMs, and use that to build an abstract probability matrix which is the cornerstone of our model and analysis. We show in our paper that the behavior of LLMs is consistent with Bayesian learning and explain many empirical observations of the LLMs using our model.

Our key contributions include:

\begin{enumerate}
    \item A continuity theorem (Section \ref{sec:continuity}) that relates the rows of the abstract probability matrix to the embeddings of the prompt and proves a result on the continuity of the mapping between the embeddings and the multinomial distribution induced by the embedding.
    \item A demonstration (Section \ref{sec:icl-bayesian}) that text generation by the LLMs is a Bayesian learning mechanism, where the probability distribution generated for the next token prediction is the posterior induced by the prompt as the new evidence, and the pre-trained model as the prior.
    \item An explanation for the emergence of in-context learning in larger models (Section \ref{sec:icl-bayesian}).
    \item Empirical validation using visualizations of next token probabilities from an instrumented Llama model (Section \ref{sec:experiments}).
    \item A comprehensive analysis of embeddings and general LLM operation (Section \ref{sec:embeddings-algebra}), which provides insights into how LLMs navigate the high-dimensional embedding space during token prediction and generation.
\end{enumerate}

In the next section we describe the paper organization.

\section{Related Work}\label{sec:related}
Several recent works have explored connections between Bayesian principles and modern machine learning techniques. Xie et al.~\cite{DBLP:journals/corr/abs-2111-02080} frame in-context learning in large language models as implicit Bayesian inference, showing how it can be viewed as inference in a mixture of hidden Markov models. While they focus specifically on in-context learning, our work provides a broader Bayesian framework for understanding LLM behavior, including but not limited to in-context learning. Fong and Holmes~\cite{10.1093/biomet/asz077} demonstrate an equivalence between the marginal likelihood and exhaustive cross-validation using the log posterior predictive as a scoring rule. Our work extends this connection to the that of LLMs, showing how next token prediction in LLMs can be understood through a similar Bayesian lens. Moreno-Muñoz et al. ~\cite{10.5555/3666122.3669615} prove that masked pre-training in large language models implicitly maximizes the model's marginal likelihood. In contrast, our work focuses on the generative process of LLMs rather than their training, providing a theoretical foundation for understanding LLM behavior during inference.

\subsection{Paper organization and our contributions}

We first describe our approach at a high level, and in the rest of the paper get into the details of the approach. We focus on the optimization metric of these LLMs, namely, predict the next token, and develop the model from there on. We first describe the ideal generative text model (Section~\ref{sec:ideal-text}), and relate it to its representation of an abstract (and enormous) multinomial transition probability matrix. We argue that the optimization metric results in these LLMs learning to represent this probability matrix during training, and text generation is nothing but picking a multinomial distribution from a specific row of this matrix. This matrix, however is infeasible to be represented by the LLMs, even with billions of parameters, so the LLMs learn to approximate it. Further, the training data is a subset of the entire text in the world, so the learnt matrix is an approximation and reflection of the matrix induced by the training data, rather than the a representation of the ideal matrix. 

Next (Section~\ref{sec:embeddings}), we relate the rows of this matrix to the embeddings of the prompt and prove (Theorem~\ref{thm:continuity}) a result on the continuity of the mapping between the embeddings and the multinomial distribution induced by the embedding. We then prove (Theorem~\ref{thm:dirichlet}) that any prior over multinomial distribution can be represented as a finite mixture of Dirichlet distributions. 

We then argue, and demonstrate (Section~\ref{sec:icl-bayesian}) that text generation by the LLMs is a Bayesian learning mechanism, where the probability distribution generated for the next token prediction is the posterior induced by the prompt as the new evidence, and the pre-trained model as the prior. Next, we develop our model and theory further in Section~\ref{sec:bayesian-model} using a real world example of in-context learning. 

To empirically validate our model via experimentations in Section~\ref{sec:experiments}, we have instrumented the open-source Llama model to visualize next token probabilities during inference tasks. These visualizations provide empirical support for our Bayesian learning framework, demonstrating how LLMs progressively adapt their token distributions based on input prompts and examples. 

In Section~\ref{sec:embeddings-algebra}, we expand our discussion beyond in-context learning to explain the general operation of LLMs using the concept of embeddings. We provide an abstract representation of the embeddings that an LLM works with, encompassing both the embeddings of the prompt and the embeddings from the pre-trained model. This section provides an explanation on how LLMs navigate the high-dimensional embedding space during text generation. We discuss the process of distribution generation, contrasting in-context learning scenarios with general text completion. We also explain the token-by-token generation process and its implications for LLM behavior, including adaptability, coherence, contextual understanding, and limitations such as prompt sensitivity and hallucination.

We finally present some implications of our model (Section~\ref{sec:implications}) and conclude with open questions and directions to pursue (Section~\ref{sec:conclusions}).

In the next section we start with a model for text generation/occurrence in the real world.

\section{Text generation in the real world}

\subsection{The ideal generative text model}\label{sec:ideal-text}
We start with an abstract model of the \textit{entire} written text/knowledge that exists in the world, which we denote as $\mathcal{S}$. There is a finite vocabulary of the text, which we denote as $\mathcal{M}$, of size $|\mathcal{M}|=m$. Every word occurs with some probability $p_i$ in this corpus $\mathcal{S}$, with a multinomial distribution over the vocabulary $\mathcal{M}$, $(p_1, p_2, \cdots, p_i, \cdots, p_m)$ with a prior $u$. Suppose we pick a word randomly from this distribution, and say this word is ``Protein". Now, again there will be a generated multinomial distribution $u(\cdot |\text{``Protein"})$ over the vocabulary $\mathcal{M}$ given that the first word is Protein. We denote this multinomial distribution as U(``Protein"). This multinomial distribution will be sparse (a very small subset of words will follow ``Protein" over the vocabulary $\vocab$), and two words that will likely have a non-negligible probability are ``synthesis" and ``shake". If we sample the next word according to this multinomial distribution U(``Protein"), we will generate a conditional (conditioned on which word is picked) multinomial distribution U(``Protein synthesis") or U(``Protein shake") etc. The multinomial distribution for U(``Protein synthesis") will be dominated by terms related to biology, whereas the  multinomial distribution U(``Protein shake") will be dominated by terms related to exercise and gym. And we keep progressing down this tree, as depicted pictorially in Figure~\ref{fig:text-in-real-world}.
\begin{figure*}[htb]
    \centering
    \includegraphics[width=\textwidth]{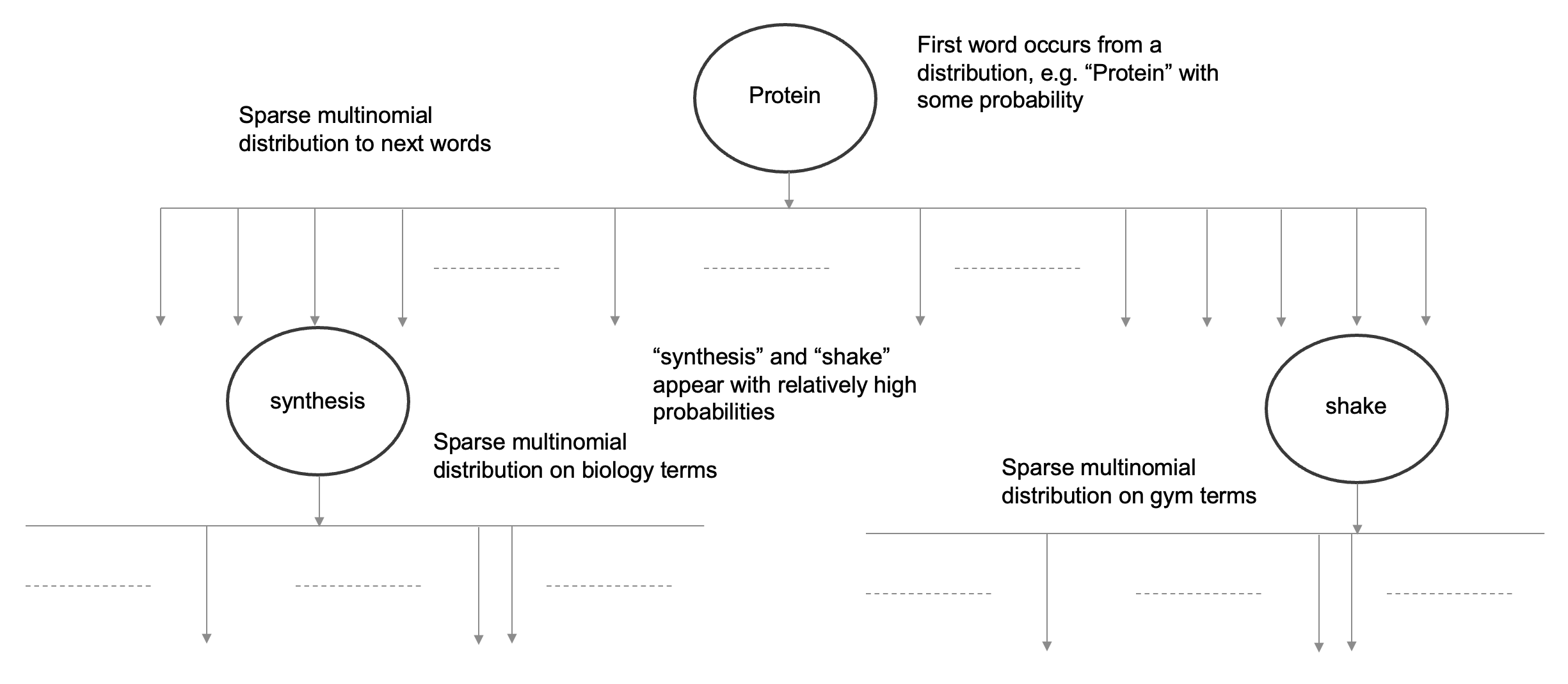}
    \caption{A conceptual model for text in the real world}
    \label{fig:text-in-real-world}
\end{figure*}

Now we can view the entire text corpus $\corpus$ generating different multinomial probabilities for each sequence of words (or ``prompts" as they are commonly referred to). If we consider a typical large language model (LLM) like ChatGPT, they will have a vocabulary of say 50,000 tokens (tokens are words/sub-words), and a prompt size that they respond to be about 8000 tokens. This results in a multinomial probability matrix of size $50000\times50000^{8000}$ as depicted in Figure~\ref{fig:prob-matrix}, where each row corresponds to a unique combination of 8000 tokens and each column is a token in the vocabulary of the LLM. This matrix is \textit{enormous}, more than the number of atoms across all galaxies. Fortunately, it is very sparse in practice as an arbitrary combination of tokens is likely gibberish, and occurs with 0 probability. Even for rows that occur with non-negligible probability, the column entries in that row are also very sparse, since most entries of the multinomial distribution would be zero (``Protein synthesis" is unlikely to be followed by ``convolutional networks" etc.). However, even with the row and column sparsity, the size remains beyond the capacity to accurately represent and so practical generative text models are built on several approximations, and we will go over them in the next section. At a fundamental level, LLMs are trying to compactly represent this Probability matrix, and given a prompt, they attempt to recreate the multinomial distribution in the row corresponding to the prompt.
Note that these LLMs are trained to ``predict the next token (accurately)'' as the (optimization) objective. With that objective, the loss function used during training is the cross entropy loss function. It is straightforward to show that in the ideal scenario, the optimal multinomial distribution that they generate $\hat{u}(\cdot| \text{``Prompt"})$, should match the empirical multinomial distribution $u(\cdot| \text{``Prompt"})$ that exists in the training corpus $\corpus$, since the cross entropy $H(p,q)$ is minimized when $p\equiv q$, i.e. the generated and empirical distributions match. However, as stated earlier, this ideal is impossible to achieve in practice. In the next section we look at how LLMs work, and the approximations that are involved in practical settings.

\begin{figure}[htb]
    \centering
    \begin{tikzpicture}[scale=0.7, transform shape]
        \draw (0,0) rectangle (4,8);

        \node[anchor=west] at (4.5,4) {$50,000^{8000}$ rows};
        \node[anchor=south] at (2,8.5) {$50,000$ columns};

        \matrix (mat) at (0.30,4.6) [row sep=0.1cm,column sep=0.1cm] {
            {\tiny ${p}_{i1} {p}_{i2} \ldots \ldots \ldots \ldots {p}_{im}$} & & & \\
            & & & \\
            & & & \\
            & & & \\
        };


        \node[anchor=east] at (-0.5,2.4) {``The cat sat on the.."};
        \node[anchor=east] at (-0.5,4.6) {``Protein synthesis"};
        \node[anchor=east] at (-0.5,5) {``Protein shake"};

        \node at (2.5,0.5) {$\cdots$};
        \node at (0.5,3.5) {$\vdots$};
        \node at (2.8,3.8) {$\ddots$};
    \end{tikzpicture}
    \caption{Probability Matrix}
    \label{fig:prob-matrix}
\end{figure}
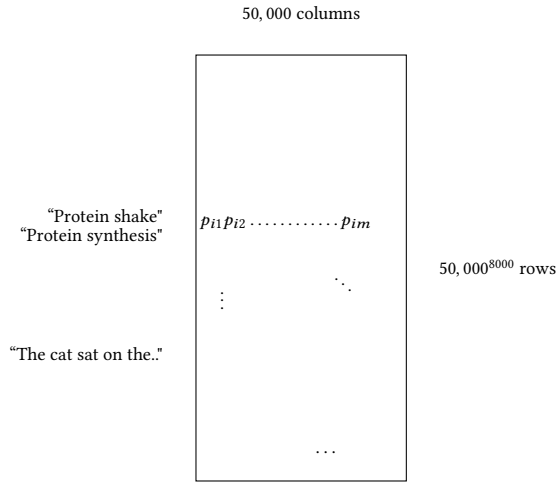

\subsection{Real world LLMs}

In an abstract sense, Large Language Models (LLMs) work by using a given prompt to locate a specific row in the probability matrix. From this row, they extract a multinomial distribution, which then guides the selection of the next token by sampling from this distribution. As an example, for the prompt ``The cat sat on the", the LLM generates a multinomial distribution as shown in Figure~\ref{fig:multinomial}. The tokens ``mat" and ``couch" have the highest probabilities, and tokens ``courage" or ``the" have (extremely) low probabilities. This token is added to the prompt, and the process repeats, with the updated prompt leading to a new row in the matrix, continuing token generation in a sequential manner. 
\begin{figure}[htb]
    \centering
    \includegraphics[width=0.6\linewidth]{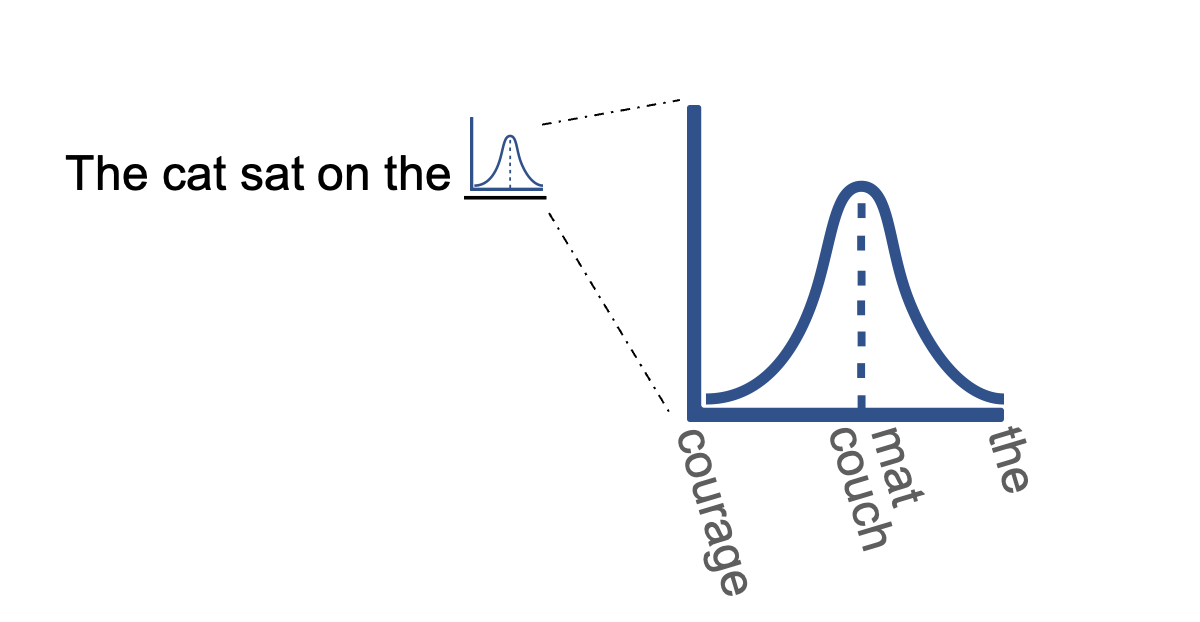}
    \caption{Multinomial distribution produced by an LLM in response to a prompt}
    \label{fig:multinomial}
\end{figure}

The perfect probability matrix contains rows for all the text that is found (or can be generated) in the world, however LLMs can only create an approximation of it using the training set, which is a subset,
$\traincorpus$, of the full corpus 
$\corpus$. The behavior of an LLM depends on the selection of 
$\traincorpus$. So the first approximation that affects the performance of the LLMs is the incompleteness of the training set. A second approximation involves the the representation of the matrix that is generated from training on this incomplete set. 

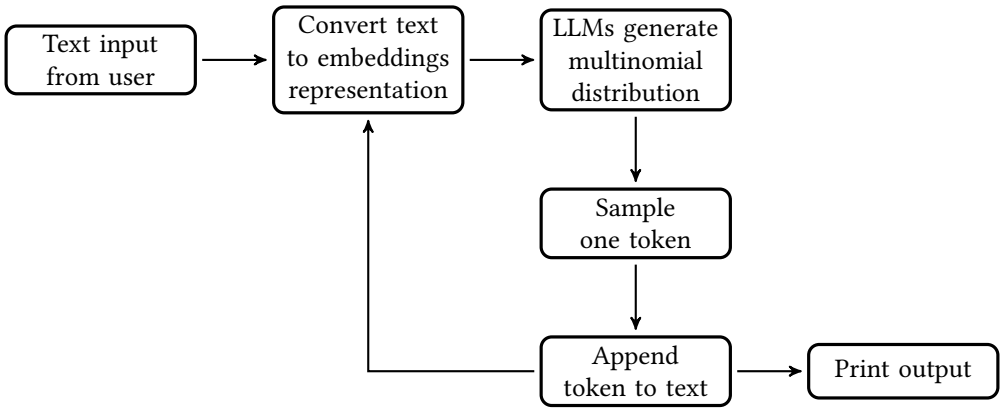
\begin{figure*}[ht!]
\centering
\begin{tikzpicture}[node distance=1cm, auto]
 \node[punkt] (input) {Text input from user};
 \node[punkt, right=of input] (embeddings) {Convert text to embeddings representation};
 \node[punkt, right=of embeddings] (llm) {LLMs generate multinomial distribution};
 \node[punkt, below=of llm] (sample) {Sample one token};
 \node[punkt, below=of sample] (append) {Append token to text};
 \node[punkt, right=of append] (print) {Print output};
 
 
 \draw[pil] (input) -- (embeddings);
 \draw[pil] (embeddings) -- (llm);
 \draw[pil] (llm) -- (sample);
 \draw[pil] (sample) -- (append);
 \draw[pil] (append) -- (print);
 \draw[pil] (append.west) -| (embeddings.south);
\end{tikzpicture}
\caption{Functional blocks of an LLM based generative text system}
\label{fig:llm_architecture_feedback}
\end{figure*}
The other approximation comes from representing text as embeddings. 
A short primer on how LLMs like gpt use embeddings for representation is as follows:
LLMs begin processing text by converting sequences of characters or tokens into a fixed-dimensional space where each unique token is represented by a high-dimensional vector, known as an embedding. This representation captures semantic and syntactic properties of the language, enabling the model to understand the contextual relationships between tokens. 

For instance in the transformer ~\cite{NIPS2017_3f5ee243} architecture the embeddings serve as the input layer, which employs an attention~\cite{bahdanau2015neural} mechanism to weigh the influence of different parts of the input text when predicting the next token. The attention mechanism allows the model to focus on relevant tokens for each prediction step, regardless of their position in the input sequence, thereby enabling the handling of long-range dependencies and variable-length input. This representation is then used downstream to generate the multinomial output distribution corresponding to the text input, however for our model in this paper we will abstract out the specifics of an architecture like the Transformer, and only assume that the input to the architecture is an embedding vector representing the prompt.

A decomposition of the functional blocks of an (autoregressive) LLM based generative text model is shown in Figure ~\ref{fig:llm_architecture_feedback}. Text is entered by the user as a prompt, it is converted into embeddings by the LLM, then the LLM processes the embeddings as an input, produces an output multinomial distribution based on the embeddings and samples the next token from this distribution. The next token is appended to the prompt, converted into embeddings again and the process repeats until the next token picked corresponds to ``end of response".

The key to understanding how In-Context learning or in general text generation works is to analyze how the networks respond to prompts and is similar to the question of out of domain generalization capabilities of classifiers in deep learning~\cite{Kawaguchi_2022}. In the next sections we argue and demonstrate that all text generation in LLMs is consistent with a form of Bayesian learning, and In-Context learning is a special case of that.

\section{Embeddings and Prior Approximations}\label{sec:embeddings}
\subsection{Preliminaries}
In our model, each prompt has the corresponding representation in its embedding. Let $\mathscr{E}$ be the space of embeddings. For example  $\mathscr{E} = \mathscr{R}^r$. We observe a finite number of embeddings, say $e_1,...,e_n$ and each $e_i$ is mapped to a next token multinomial probability vector of the size of the vocabulary $m$,  say $({p_e}_{i1},...,{p_e}_{im})$, ${p_e}_{ij} \geq 0, \sum {p_e}_{ij} =1$. Let the space of such probability vectors be $\mathscr{P}$. 
We assume $\mathscr{E}$ is a metric space. 
 
 Suppose  $T$ maps the embeddings to $\mathscr{P}$ by a convexity preserving transformation. That is, $T(\alpha e_1 + (1-\alpha ) e_2) = \alpha T(e_1) + (1- \alpha) T(e_2)$. Consider the $\mathscr{L}_p$ metric for any $p \geq 0$ on $\mathscr{P}$. $T$ is clearly bounded in this metric by $1$. 

 \subsection{Continuity of Embedding Mapping:}\label{sec:continuity}
 \begin{theorem}[Continuity] \label{thm:continuity}
     If the mapping $T$ is convexity preserving and bounded, then it is continuous.  \end{theorem}

\begin{proof}
 Consider any two points $x$ and $y$ in $\mathscr{E}$. Define  $x_{\alpha} = \alpha y + (1-\alpha ) x$. This defines a ray from $y$ to $x$. The corresponding $T$ mapping by the convexity preserving property is $\alpha T(y) + (1- \alpha ) T(x)$. Clearly as $\alpha \rightarrow 0, \, x_ \alpha \rightarrow x$ and $T(x_{\alpha}) \rightarrow T(x) $ by the boundedness of $T(y)$. So the continuity is established along every ray. Now consider any arbitrary sequence $(x_n, n=1,2,...) \rightarrow x$, then each of the point is on some ray to $x$. Thus establishing continuity for any sequence.
\end{proof} 

The above theorem allows us to approximate any new multinomial distribution induced by an unseen embedding with the multinomial distributions induced by known embeddings as long as the operation is convex linear combination; for example by the nearest $k-means$ procedure. 

Note that while we make this assumption about the convexity preservation mapping of embeddings to the multinomial distribution which leads to our continuity theorem, this property is important for the well-posedness of the posterior distribution in Bayesian statistics with respect to measurement errors, and has been shown in ~\cite{dolera2023lipschitz}. Additionally, the convexity preserving property can be viewed as picking one embedding with probability $\alpha$, and the other with probability $1-\alpha$, and the linearity of expectation implies that the associated distributions induced by those embeddings also preserve the same weighting in expectation. Informally, this property leads to well-behaved LLMs that don't have ``wild'' outputs.

\subsection{Prior Approximation:}
In the rest of the paper, we will be using Bayesian priors for the multinomial probabilities in the matrix. We now state a general result that simplifies computation by approximating any arbitrary prior by a mixture of Dirichlet distributions. The advantage of this approximation is that Dirichlet distributions are natural conjugate priors for multinomial and thus analytic calculations are easy to carry out for them and their mixtures.  The proof is rather technical and is deferred to the appendix.

\begin{theorem}[{Dirichlet approximation}] \label{thm:dirichlet} Any continuous bounded prior over $multinomial$ probabilities, $u(p_{1},p_{2},\cdots p_{m})$, can be approximated as a finite mixture of Dirichlet distributions $\mathcal{D}$.

\end{theorem}

That is,
\begin{equation}
u(P)\approx \sum_{k=1}^{n} b_{k,n}\mathcal{D}((p_1,p_2,\dots,p_m)|(\alpha_{k,1},\alpha_{k,2},\dots,\alpha_{k,m}))
\end{equation}
for some mixing constants $b_{k,n} \geq 0, \sum_k b_{k,n}=1$, and parameters $(\alpha_{k,1},\alpha_{k,2},\dots,\alpha_{k,m})$. For details see the Appendix.  Further, the posterior after observing any multinomial outcomes is also a mixture with reweighted $b's \, and \, \alpha's$. 

It is important to relate \textit{Theorems 1} and \textit{2} in that if a mapping $T$ satisfies \textit{Theorem 1}, and leads embeddings $E_i \in \mathscr{E} $ to probability distributions  $P_i \in \mathscr{P}$ that are mixtures of Dirichlet,  then any convex linear combination of the embeddings under $T$ will also lead to a mixture of Dirichlet. That is a convex hull of the class of embeddings will be mapped to a convex hull of the class of mixtures of Dirichlet. More critically, the same thing will happen with the posterior embeddings and the posterior probability distributions. In the next sections for illustrative purposes we typically will use Dirichlet distributions with the understanding that the results generalize to Mixtures and consequently to arbitrary priors.

\section{Text generation and Bayesian learning}
We argue that text generation by LLMs is consistent with a Bayesian learning process. When an LLM is fed a prompt, it goes through two steps. First, whatever current representation it has stored of the matrix, it locates the embeddings ``closest'' to the the embedding of the prompt, and the approximation via Theorems ~\ref{thm:continuity} and ~\ref{thm:dirichlet} of the multinomial distribution serves as the prior for the Bayesian learning. Next, the embedding of the prompt itself is treated as new evidence (likelihood) and the two are used to compute the posterior which is then used as the multiniomial distribution for the next token prediction. Note that if the prompt is an embedding the LLM has been trained on, then the Bayesian learning simply returns the prior distribution as the posterior (this is also the most efficient learning process during training, to minimize cross entropy loss). When the prompt contains something ``new'', the posterior adjusts to this new evidence. This process is depicted in Figure ~\ref{fig:bayesian-updating}. How efficiently and accurately the posterior adjusts depends on the size of the LLM, and in the next subsections we show that In-Context learning within LLM models is consistent with Bayesian learning. 
\begin{figure*}[htbp]
\centering
\begin{tikzpicture}[node distance=1.5cm and 1.5cm, auto, box/.style={draw, rounded corners}]
    \node[box] (prompt) {Prompt};
    \node[box, right=of prompt] (embedding) {Embedding};
    \node[box, below=of prompt] (likelihood) {Likelihood};
    \node[box, above right=of embedding, text width = 3.7cm,  xshift=-0.5cm] (llm) {LLM pre-trained model, \\ Theorems~\ref{thm:continuity} and \ref{thm:dirichlet}};
    \node[box, below=of llm] (approx) {Approximate Prior};
    \node[box, below =of approx] (bayes) {Bayes update};
    \node[box, right=1.5cm of bayes] (posterior) {Posterior Distribution};

    \draw[-stealth] (prompt) -- (embedding);
    \draw[-stealth] (prompt) -- (likelihood.north);
        \draw[-stealth] (embedding) |- (llm.west);

    \draw[-stealth] (likelihood.east) |- (bayes);
    \draw[-stealth] (llm.south) -- (approx.north);
    \draw[-stealth] (approx.south) -- (bayes);
    \draw[-stealth] (bayes) -- (posterior);
\end{tikzpicture}
\caption{Conceptual Bayesian updating of next token multinomial probability}
\label{fig:bayesian-updating}
\end{figure*}
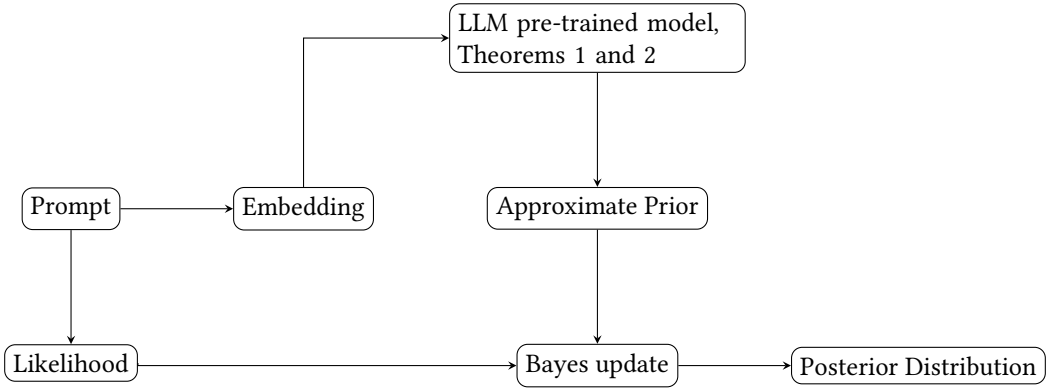

\subsection{In Context Learning - Preliminaries}
In-Context learning is a technique whereby task-specific responses are generated by giving task specific prompts to a LLM. There are many ways to do it, either by 0 shot or few shot learning. In-Context Learning can be classified in three broad categories according to~\cite{wei2023larger}:

\begin{itemize}
\item Regular In Context Learning
\item Flipped-Label In-Context Learning
\item Semantically Unrelated In-Context Learning (SUIL)
\end{itemize}

\begin{figure}[htb]
    \centering
    \includegraphics[width=0.6\linewidth]{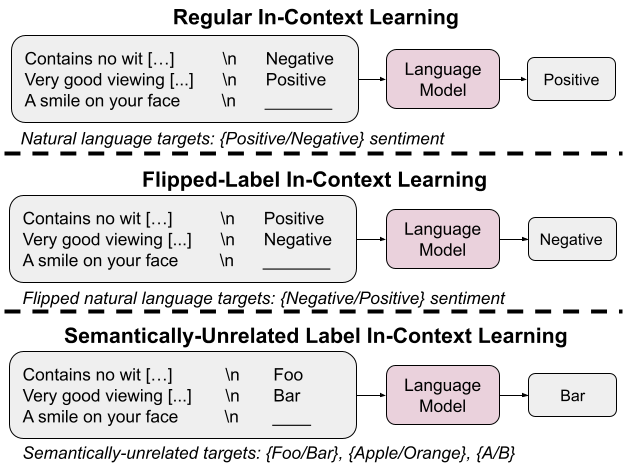}
    \caption{Types of In-Context Learning- from~\cite{wei2023larger} }
    \label{fig:enter-label}
\end{figure}

Basically, compared to the regular In-Context Learning, Flipped-Label In-Context Learning allows LLM prompts to have flipped labels so that the labels are opposite to the pre-trained LLM model. For example in the sentiment analysis task, ``positive" is flipped to ``negative" while prompting and vice versa. In Semantically Unrelated In-Context Learning (SUIL), ``positive" is converted for example to ``foo" and ``negative" is converted to ``bar"; both semantically unrelated in the pre-trained model. 

What is surprising is that LLMs are able to deal with these inconsistencies and are able to adapt to the new labels rather quickly depending upon the number of parameters of the LLM. In the following sub-section we introduce the Bayesian paradigm using Dirichlet distributions and show how this behavior of LLM is consistent with the Bayesian Learning. 

\subsection{Bayesian Learning is In-context learning}\label{sec:icl-bayesian}
We start with the simpler case of SUIL with two labels $A \, B$, where only one label is changed from $A$ to $B$.  Since at any stage the auto-generation is done by the $multinomial$ distribution corresponding a row of the matrix, if we were to just consider the occurrence of the labels $A$ and $B$, the corresponding distribution will be a $Binomial(n,p)$ where $n$ is the sample size, and $p$ is the corresponding occurrence probability. For our exposition, besides assuming the occurrence of the labels $A,B$ is $Binomial$, for the Bayesian setting we further assume that it has a $Beta$ prior. The result in this section holds irrespective of this assumption since any prior distribution over $Binomial$ probability can be approximated by mixtures of $Beta$ priors as proven in Theorem 4.1. It also indicates that the entire development can be generalized to $multinomial$ distribution and $Dirichlet$ prior and mixtures of $Dirichlet$ priors when multiple labels are changed. 

Recall that if $X \sim Bin(n,p)$ with the prior over $p$, $p \sim Beta (\alpha, \beta)$, and a fixed $n$, then, $$p | n,x \sim Beta(\alpha+x, \beta +n -x)$$ with the corresponding posterior mean, $E(p |n,x)=(\alpha + x)/(\alpha + \beta +n)$, and the posterior variance $Var(p|x,n)=((\alpha + x)(\beta +n-x))/((\alpha + \beta +n)^2 (\alpha + \beta +n +1))$ which is $O(n^2)$. Further, $\alpha +\beta$ can be considered as the prior sample size of $A$ and $B$ , and $\alpha$ is the prior occurrence of $X_A$. See the related discussion in~\cite{dalal1983approximating}.

Now consider the most unrestricted case of the SUIL, namely, where the LLM is trained on the label $A$ with rare occurrences of the label $B$. \\

Consider now the auto-generated response conditional on $A$ occurring according to the training data of the base LLM. Then the training distribution for $A$ and $B$ can be represented as a Binomial distribution; $(p_A, p_B),\, p_A+p_B=1$ with a prior for ($ p_A, p_B)$ to be $Beta(\alpha_A,\beta_B)$. Since the training data from the LLM is mainly based on the label $A$ with a rare occurrence of $B$, we will have $\alpha_A \gg \beta_B$. Thus, $E(p_A|n=0)=\alpha_A/(\alpha_A+\beta_B) \approx 1$ and $E(p_B|n=0)=\alpha_B/(\alpha_A+\beta_B) \approx 0$. Further, since LLMs are trained on many labels, $\alpha_A$ would be relatively small, though much greater than $\beta_B$.\\

Now consider In-Context Learning for SUIL. Here we are replacing $A$ by $B$ in $n$ prompts, thus we have $x_B=n$ prompts of $B$ and $x_A=0$ prompts of $A$, with all other context remaining the same.  In this case 
$$E(p_A|x_A,x_B)=\alpha_A/(\alpha_A+\beta_B +n), \, and,$$
$$E(p_B|x_A,x_B)=(\beta_B +n)/(\alpha_A+\beta_B +n)$$
So, clearly as the number of prompts $n$ for the label $B$ increases, $E(p_B|x_A,x_B) \rightarrow 1$ and $E(p_A|x_A,x_B) \rightarrow 0$. 

We examine the qualitative behavior of this convergence, in the following table with $\alpha =0.3 \, , \beta =0.01$. Here $\alpha_A/\beta_B =30$. As can be observed, with \textbf{only three flipped examples}, the posterior adjusts to nearly flip the probabilities of the labels from the pre-trained prior.

\begin{table}[h!]
\begin{center}
\begin{tabular}{||c c c ||} 
 \hline
 n & $E(p_A |n)$ &  $E(p_B |n)$  \\ [0.5ex] 
 \hline\hline
 0 & 0.968 & 0.032  \\ 
 \hline
 1 & 0.229 & 0.771  \\
 \hline
 2 & 0.13 & 0.87  \\
 \hline
 3 & 0.091 & 0.909  \\ [1ex] 
 \hline
\end{tabular}
\caption{Behavior of $E(p_A |n),E(p_A |n)$ with $n$ prompts and $\alpha =0.3 \, \beta =0.01$ }
\label{table:1}
\end{center}
\end{table}

A similar behavior persists if $\alpha_A \ll n$. Further, to examine the asymptotic behavior consider
$$E(p_A|n)/E(p_A|n=0)=1/(1+n/(\alpha_A + \beta_B)).$$ 
This suggests even if $\alpha_A \gg \beta_B$ resulting $E(p_B|n=0) \approx 0$, as long as $(\alpha_A + \beta_B)$ is small, the In-Context Learning even in SUIL case will be very fast. 

By analogy similar results hold for other categories of the In-Context Learning and also when multiple labels are being replaced. Finally, since  Chain-of-Thought~\cite{wei2023chainofthought} learning is a type of In-Context Learning, the same results apply to it.

\cite{wei2023larger} empirically shows that the property of the adaptation to the Flipped-label In-Context Learning and SUIL depends upon the size of the model- the larger models learn better than smaller models. The Bayesian Learning also mimics this behavior by increasing $(\alpha_A + \beta_B)$, that is increasing the prior sample size resulting in more peaked distribution around the labels. Table 2 shows with $\alpha =3 \, \beta =0.1$. Here $\alpha_A/\beta_B $ still remains $=30$. Unlike the previous example, the posterior here is slow to adjust, and the probabilities of the two labels remain nearly equal after 3 examples whereas in the previous example they had nearly flipped.

\begin{table}[ht!]
\begin{center}
\begin{tabular}{||c c c ||} 
 \hline
 n & $E(p_A |n)$ &  $E(p_B |n)$  \\ [0.5ex] 
 \hline\hline
 0 & 0.968 & 0.032  \\ 
 \hline
 1 & 0.732 & 0.268  \\
 \hline
 2 & 0.588 & 0.412  \\
 \hline
 3 & 0.492 & 0.508  \\ [1ex] 
 \hline
\end{tabular}
\caption{Behavior of $E(p_A |n),E(p_A |n)$ with $n$ prompts and $\alpha =3 \, \beta =0.1$ }
\label{table:2}
\end{center}
\end{table}

This behavior can be intuitively explained since the larger models tend to have many more parameters, thus during the training they are acquiring more general knowledge scattering the probabilities across many more labels and parameters. This will result in smaller $\alpha_A + \beta_B$ for any two labels, and our model explains the emergence of In-Context learning with larger models, as observed in ~\cite{wei2023larger} amongst others. We now move from this toy example to a more concrete, real world example of in-context learning to develop our theory.

\section{In-Context Learning With Prompts and Bayesian Updating}\label{sec:bayesian-model}
This section is motivated by a real world example related to a domain specific language (DSL) designed to interpret natural language questions related to statistics in the game of cricket and convert them into a structured output (resembling json). This particular example is chosen because this DSL is not part of training data of LLMs and so is completely unknown to them at the time of prompting. However, with a few examples, the LLMs learn the DSL in-context and are able to convert a new query into the DSL. In this section we go through this process and demonstrate the LLMs learning via the evolution in next token probabilities. Table~\ref{tab:queries_outputs} is an example of $K$ examples of the NL->DSL translation and this is followed by the query: 
``highest losing team total in Tournament0''\footnote{in the DSL actual names of players, tournaments, venues are abstracted out into generic labels}

\begin{table}[h]
\centering
\begin{tabular}{|p{0.45\textwidth}|p{0.45\textwidth}|}
\hline
\rowcolor{blue!80!black}\color{white}\textbf{Natural Language Query} & \color{white}\textbf{DSL representation} \\
\hline
\rowcolor{cyan!30}
Tournament0 team with best win loss record after losing the toss & 
\{'orderby': ['win\_loss\_ratio'], 'toss': ['lost'], 'tournament': ['Tournament0'], 'type': ['team']\} \\
\hline
\rowcolor{cyan!30}
lowest team total & 
\{'groupby': ['innings'], 'orderby': ['runs'], 'sortorder': ['reverse'], 'type': ['team']\} \\
\hline
\rowcolor{cyan!30}
biggest Tournament0 total in defeat & 
\{'groupby': ['innings'], 'orderby': ['runs'], 'result': ['loss'], 'tournament': ['Tournament0'], 'type': ['team']\} \\
\hline
\rowcolor{cyan!30}
highest scores by Team0 & 
\{'groupby': ['innings'], 'orderby': ['runs'], 'team': ['Team0'], 'type': ['team']\} \\
\hline
\end{tabular}
\caption{Few shot examples of NL->DSL}
\label{tab:queries_outputs}
\end{table}

In response the LLM generated the following completion: 
$$\{'groupby': ['innings'], 'orderby': ['runs'], 'result': ['loss'], 'tournament': ['Tournament0'],$$ $$ 'type': ['team']\} $$

Another similar example is the following: ``Person0 batting record in the power plays of Tournament0 in the Tournament1'' and the $K$ examples are shown in Table~\ref{tab:queries_outputs_extended}.

\begin{table}[htbp]
\centering
\begin{tabular}{|p{0.4\textwidth}|p{0.55\textwidth}|}
\hline
\rowcolor{blue!80!black}\color{white}\textbf{Natural Language Query} & \color{white}\textbf{DSL representation} \\
\hline
\rowcolor{cyan!30}
Team0 batting average in powerplays Tournament0 Season0 & 
{'groupby': ['nodefault', 'opposition'], 'opposition': ['Team0'], 'orderby': ['average'], 'overs\_type': ['powerplay'], 'season': ['Season0'], 'team\_view': ['bowl'], 'tournament': ['Tournament0'], 'type': ['team']} \\
\hline
\rowcolor{cyan!30}
Person0 batting record in the death overs against Person1 & 
{'batsman': ['Person0'], 'bowler': ['Person1'], 'groupby': ['bowler'], 'overs\_type': ['death'], 'type': ['batting']} \\
\hline
\rowcolor{cyan!30}
Person0 phase wise batting record in the Tournament0 & 
{'batsman': ['Person0'], 'groupby': ['nodefault', 'over\_group'], 'tournament': ['Tournament0'], 'type': ['batting']} \\
\hline
\rowcolor{cyan!30}
most runs by Person0 in the powerplays in the Tournament0 & 
{'batsman': ['Person0'], 'groupby': ['match', 'nodefault'], 'overs\_type': ['powerplay'], 'tournament': ['Tournament0'], 'type': ['batting']} \\
\hline
\rowcolor{cyan!30}
Person0 batting record in the powerplay overs in the Tournament0 & 
{'batsman': ['Person0'], 'overs\_type': ['powerplay'], 'tournament': ['Tournament0'], 'type': ['batting']} \\
\hline
\rowcolor{cyan!30}
Person0 against Person1 in each season of Tournament0 & 
{'groupby': ['nodefault', 'season'], 'opponent': ['Person1'], 'player': ['Person0'], 'tournament': ['Tournament0'], 'type': ['primaryrole']} \\
\hline
\rowcolor{cyan!30}
highest batting score all out in Tournament0 & 
{'event': ['All Out'], 'groupby': ['innings'], 'orderby': ['runs'], 'tournament': ['Tournament0'], 'type': ['team']} \\
\hline
\rowcolor{cyan!30}
Person0 in the powerplay in the Tournament0 Season0 against Season1 & 
{'groupby': ['season'], 'overs\_type': ['powerplay'], 'player': ['Person0'], 'season': ['Season0', 'Season1'], 'tournament': ['Tournament0'], 'type': ['primaryrole']} \\
\hline
\end{tabular}
\caption{Natural Language Queries and their DSL representations}
\label{tab:queries_outputs_extended}
\end{table}

In response the LLM generated the following completion: 
$$\{'batsman': ['Person0'], 'groupby': ['nodefault', 'over\_group'], 'overs\_type': ['powerplay'], $$ $$ 'tournament': ['Tournament0', 'Tournament1'], 'type': ['batting']	\}$$

Looking at the above exchange, the structure is evident. The task consists of a sequence of queries followed by their answers. The final query is to be completed by the LLM based on learning from three sources:
\begin{enumerate}
    \item The previous queries
    \item Their answers
    \item The original matrix $\tau$, which consists of multinomial probabilities generated by the LLM's training data
\end{enumerate}

Some further notations are needed to describe DSL. Let us assume that we have $K$ query-answer pairs, $(q_i,a_i), i=1, \dots , K$, are pre-cursor to a final query $Q$ which is to be answered by LLM as $A$. Let $q_i$ consists of distinct tokens  $(t_{i,1}, \dots, t_{i,r_i})$ while the corresponding answer $a_i$ consists of tokens $(s_{i,1}, \dots, s_{i,r_i})$. Further, let $\mathcal{Q} = \cup_i q_i=\{t_1,\dots, t_r\}$, $\mathcal{A} = \cup_i a_i=\{s_1,\dots,s_r\}$. 

For simplicity of exposition, we make the following assumption. 

\textbf{Assumption 1:} For a given  $t$ in a query $q$, there is corresponding $s_t$ in the corresponding answer $a$. Further,  $Q \subset \mathcal{Q}$, $A \subset \mathcal{A}$ and every pair $(t,s)$ belongs to at least one $(q_i,a_i)$ pair. We also assume that tokens in $Q$ are distinct. 

\textbf{Comment 1:} Clearly, each of the answer $a_i$ is the corresponding query dependent. If not, then nothing can be learnt from ($q_i, a_i$) pairs. In effect there is a mapping from $t$ to $s$ which is created by the prompt designer. Further, $t$ and $s$ can be compound sub-tokens for simplicity of presentation. For example, in q1, "most sixes" is being converted in a1 to 'orderby':['sixes'], 'Tournament0' is converted in 'tournament':["Tournament0"], etc.  

\textbf{Comment 2:} $Q$ can have tokens which are not in $\mathcal{Q}$. In that case the corresponding token needs to be replaced by a different token closest to it from the embedding space per Theorem 1, e.g., $'\textbf{defeat}' \approx ~'\text{loss}'$. Also, it may be necessary to remove  \textit{stopwords} like $'\text{the}'$. For the rest of the discussion for simplicity we assume there is an exact match between every $(t,s)$ in at least one $(q_i,a_i)$.

Under the Assumption 1, $A$ is to be learnt from $(q_i,a_i)$ pairs and $Q$ based on LLM training data. According to the Bayesian paradigm we have described before, we have

\textbf{Assumption 2:} $\mathcal{Q}=\{t_1,\dots, t_r\}$ is distributed as multinomial with probabilities $\{p_1,\dots, p_r\}$, and a prior $\mathcal{D}(\alpha_1,\dots,\alpha_r)$, where, $\alpha_i$ are based on the LLM trained data coming from $\tau$. 

\textbf{Comment 3:} One can have an arbitrary prior $u(p)$, which by our section on Dirichlet approximation can be approximated by a mixture of Dirichlet distributions. For  simplicity of exposition and without a loss of generality, we assume that $u(p)$ is a Dirichlet distribution 
$\mathcal{D}(\alpha_1, \alpha_2,\dots )$ where $\alpha_i$ are based on the pre-trained model. 

In the following, 
since $Q$ has to learn from relevant prompts $q_i$, we need to find $q_i$ nearest to $Q$ by computing the probability $q_i$ given $Q$ and then some other $q_j$ with the residual of $Q$, etc. The following allows to compute any such probabilities in a more general setting which will be repeated applied.   

\begin{theorem} Let $T$ and $T*$ be sequences of tokens in pre-trained LLM ($\tau$ and Dirichlet prior), and let $T \cup T*$ consist of tokens $t_1,\dots,t_v$ with the Dirichlet prior parameters $\alpha_1,\dots,\alpha_v$ for some $v$. Then
$$P(T* | T)=\frac{\prod_{t \in T* \cap T}(\alpha_t +1)\prod_{t \in T* \cap T^c}
\alpha_t}{\prod_{j=0}^{|T*|-1} (\alpha^* + j +|T|)} \, 
where, \, \alpha^*=\sum_{i=1}^m \alpha_i \, and \, |T|=cardinality \,of \, T$$. 
\label{thm: generative probability}
\end{theorem} 

\begin{proof}
Since we are computing the probability $P(T* | T)$ we just need to look at the joint distribution of tokens in $T*$ and $T$, which is Dirichlet with parameters $\alpha_1,\dots,\alpha_v$. Suppose $T*=\{t_1*,\dots,t_K*)\}$. Then $P(T* | T)=P\{(t_1*,\dots,t_K*)|T\}$, the proof follows by conditioning $t_i*$ successively on the previous $t*$'s and $T$, and taking product using the corresponding posteriors from Dirichlet prior corresponding to $\tau$, 
$\mathcal{D}(\alpha_1,\dots,\alpha_v)$.
\end{proof}

In the following, we will iteratively use the above theorem to $q_i$'s with different $T$'s and propose a greedy Bayesian algorithm 
to construct $A$. Specifically, we decompose $Q$ and disjoint sequences based on $q_i$s from which $Q$ can learn fastest to construct $A$. The construction is prescribed below.

\textbf{Algorithm for Decomposing $Q$:}
\begin{enumerate}
    \item Let ${q^*}_1= \argmax_i P(q_i|Q)$ with the answer ${a^*}_1$. Define $Q_1= Q \cap {q^*}_1$ and ${Q_1}^c =Q-Q_1$.
    \item stop if ${Q_1}^c =\Phi$.
    \item else, 
    Let ${q^*}_2= \argmax_i P(q_i|{Q_1}^c)$ with the answer ${a^*}_2$. Define $Q_2= {Q_1}^c \cap {q^*}_2)$ and ${Q_2}^c ={Q_1}^c -Q_2$
    \item stop if ${Q_2}^c =\Phi$.
    \item Continue till for some $k$, ${Q_k}^c ={Q_{k-1}}^c -Q_k = \Phi$ and then stop.
\end{enumerate}

\begin{theorem} 
If $Q_i$ be as per the construction above till stopping at some $k$, then
$Q=\bigcup_{i \leq k} Q_i$, and $Q_i$ are mutually exclusive. 
\end{theorem}

Clearly, the selection of $q_i$'s above will depend upon the prior distribution provided by $\tau$. Finally, having found this decomposition, it is easy to construct $A$, by using the correspondence between $t$'s and $s_t$'s as below. 

\textbf{Algorithm for Constructing $A$}
\begin{enumerate}
    \item Let $A_i = {a^*}_i \cap \{ s : s \text{ corresponding to } t \in Q_i \}$, for $i=1, \dots, k$
    \item $A = \bigcup_{i=1}^k A_i$
\end{enumerate}

\textbf{Comment 4:} Note that we can compute generative probability of any sequence of $s_t$ in $A$ given $Q$ and $\tau$ by using Theorem~\ref{thm: generative probability} due to the correspondence between $t$ and $s_t$. Further note that if $Q  \not\subset \mathcal{Q}$, then the In-Context learning will not work and generate incorrect answers.

In the next section, we experimentally validate our model of learning with a real world LLM, llama-3.

\subsection{Experiments: Visualizing the learning process}\label{sec:experiments} In this section, we visualize the process of learning by LLMs. To do so, we instrumented the publicly available \textit{llama.c} code along with the llama-3 model to depict the next token probabilities visually. Our visualizer\footnote{We will make our tool open source and publicly available} lets the user inspect the next token probability distribution of not only the completion, but also of the prompt itself by hovering over the tokens. We also colorize the tokens themselves based on their probability, with green being the highest probability and red being the lowest. We depict it in Figure~\ref{fig:nyt-article} by asking the model llama-3 to complete the following paragraph from a very recent story in the New York Times (and hence outside of the training set): \textit{``A sell-off in markets around the world turned into a rout on Monday as investors grew panicky about signs of a slowing American economy, with stocks tumbling across Asia and Europe.
The moves marked a sharp reversal in the world’s major markets, which for much of the past year have risen to new heights, ''}
As we can observe, the model generates a reasonable completion (\textit{``driven by optimism about the global economy and central banks’ efforts to stimulate growth.''} vs the actual which was \textit{``propelled by optimism about cooling inflation, solid labor markets and the promise of artificial intelligence technology.''}). The probabilities can also be seen to be in the orange to red levels, because it is a distribution over the entire language and the story can go in many different ways, except in cases where grammar or world model constraints make the LLM very confident of what the next token should be, e.g., the model is very confident that ``the world'' should follow ``around'' or ``reversal'' should follow ``sharp'' etc. We also show the top 2 probabilities for the next token after ``optimism about'' and the choices are ``global'' with probability .71 and ``US'' with probability .107 (with the model picking ``global'').
\begin{figure}
    \centering
    \includegraphics[width=0.9\linewidth]{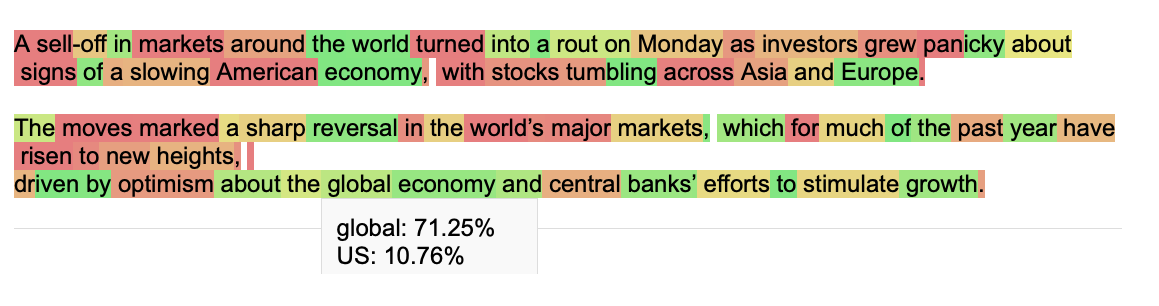}
    \caption{Visualizing the completion of a paragraph from a news story in the New York Times}
    \label{fig:nyt-article}
\end{figure}

We now depict the same visual in Figure~\ref{fig:visualizing few shots} for the in-context learning example of the cricket query DSL from Table~\ref{tab:queries_outputs}.  We show the natural language query, followed by the DSL, followed by another 3 such examples, finally followed by the natural language query ``highest losing team total
in Tournament0''. There are a few things to observe here: the natural language queries in all examples are again in the orange to red range, because they are a distribution over (close to) normal English language and hence the degrees of freedom for the continuing text are large. The DSL starts off in the orange-red for the first example, but by the time we reach the third example, the model has picked up the DSL and it has become confident about the mapping. The syntax of the DSL coupled with the semantics of the preceding natural language query restricts the possible options for the next token(s), and the final query is completed with high confidence (the only splotch of yellow seen is whether to pick a ``space'' or not after the comma in the json item).

\begin{figure}
    \centering
    \includegraphics[width=0.9\linewidth]{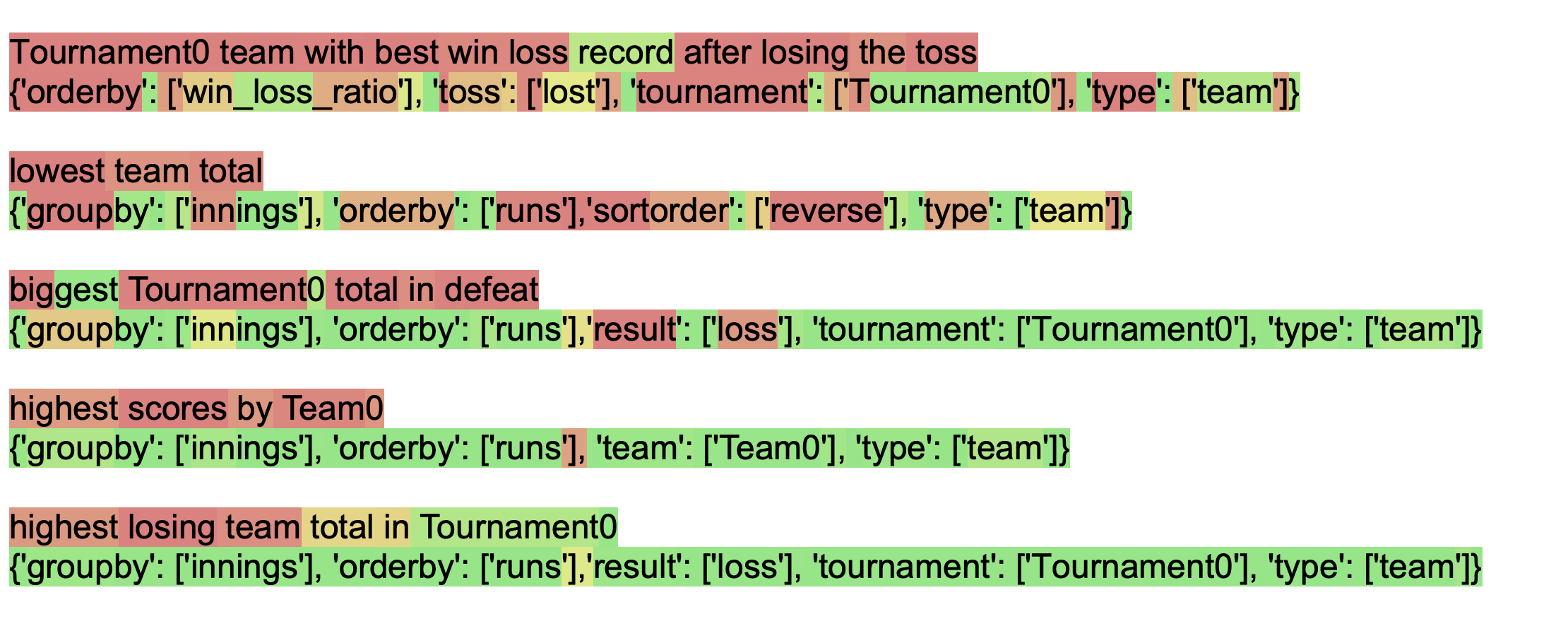}
    \caption{Visualizing in-context learning}
    \label{fig:visualizing few shots}
\end{figure}

We show the same visual for the second in-context-learning example in Table~\ref{tab:queries_outputs_extended} in Figure~\ref{fig:visualizing few shots 2} and can observe the same phenomenon. The completion by the model is confident, and the only choice that is made is in picking which order to complete the json items (overs\_type vs tournament). In this scenario we have more in-context examples and we can observe the LLM picking up on the syntax of the DSL as it goes along, with splotches of green in the earlier examples as well. 

\begin{figure}
    \centering
    \includegraphics[width=0.9\linewidth]{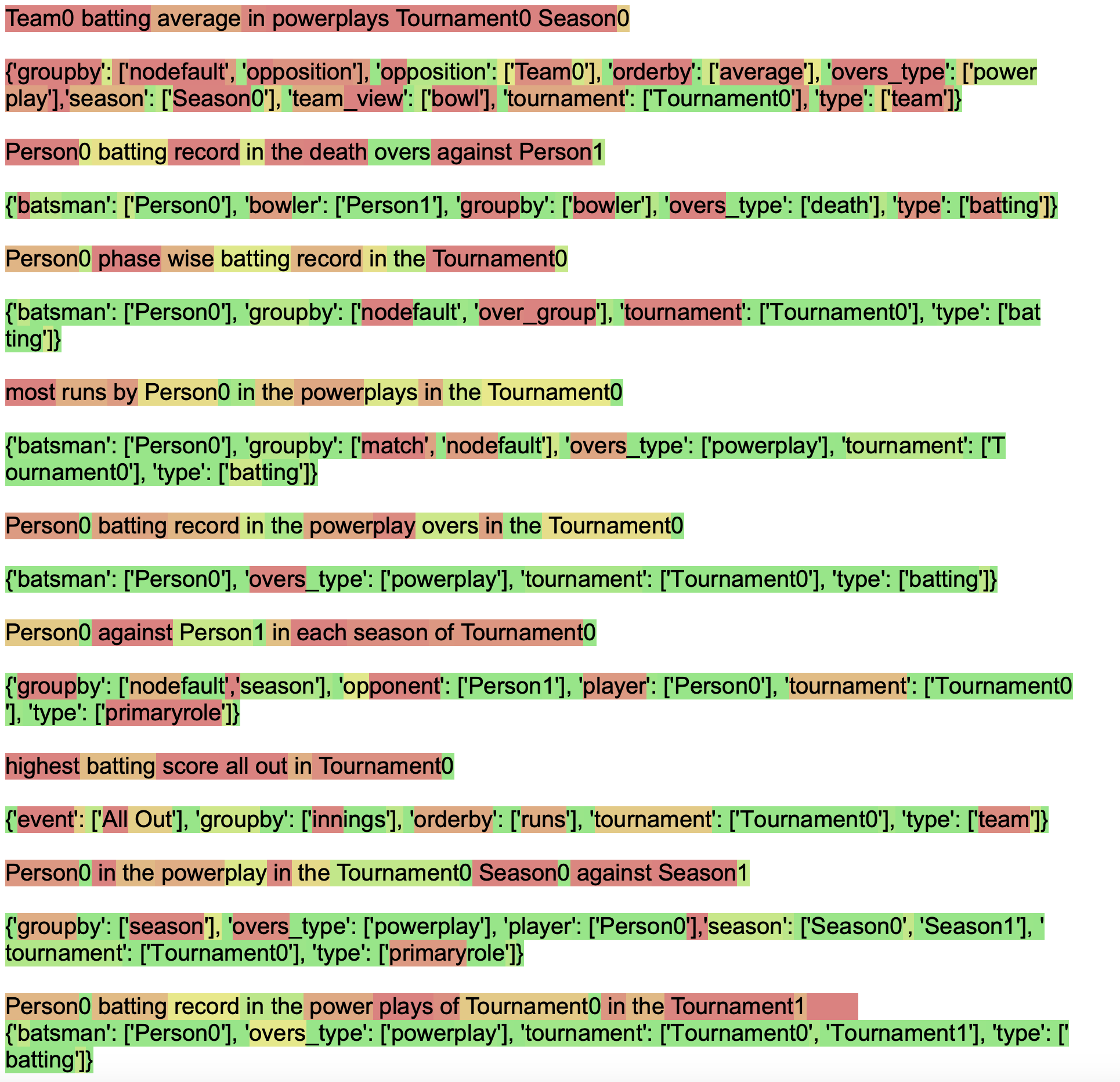}
    \caption{Visualizing in-context learning}
    \label{fig:visualizing few shots 2}
\end{figure}

Note that in Figure~\ref{fig:visualizing few shots} the completion of the query is in the third example (the new query is simply a rephrasing of the third query - ``biggest Tournament0 total in defeat'' vs ``highest losing team total in Tournament0''). If we replace the third example with something irrelevant, we can see the LLM still does a completion but it gets it wrong as the new query is not the subset of the in-context examples, which was a requirement of Assumption 1 (Also, see Comment 4). We can see that in Figure~\ref{fig:incorrect-examples} where now the third example is irrelevant and so the LLM generates ’toss’: [’lost’] instead of ’result’: [’loss’] which is the correct DSL since it has not seen an example of that in the previous $a_i$s. This points to both the importance of Assumption 1, as also the role that the pre-training plays - ’toss’: [’lost’] is the closest in english to the term "losing" that the LLM found in the examples and picked it and this kind of probabilistic generation is the cause of the well known issue of ``hallucinations'' in LLMs.

\begin{figure}
    \centering
    \includegraphics[width=0.9 \linewidth]{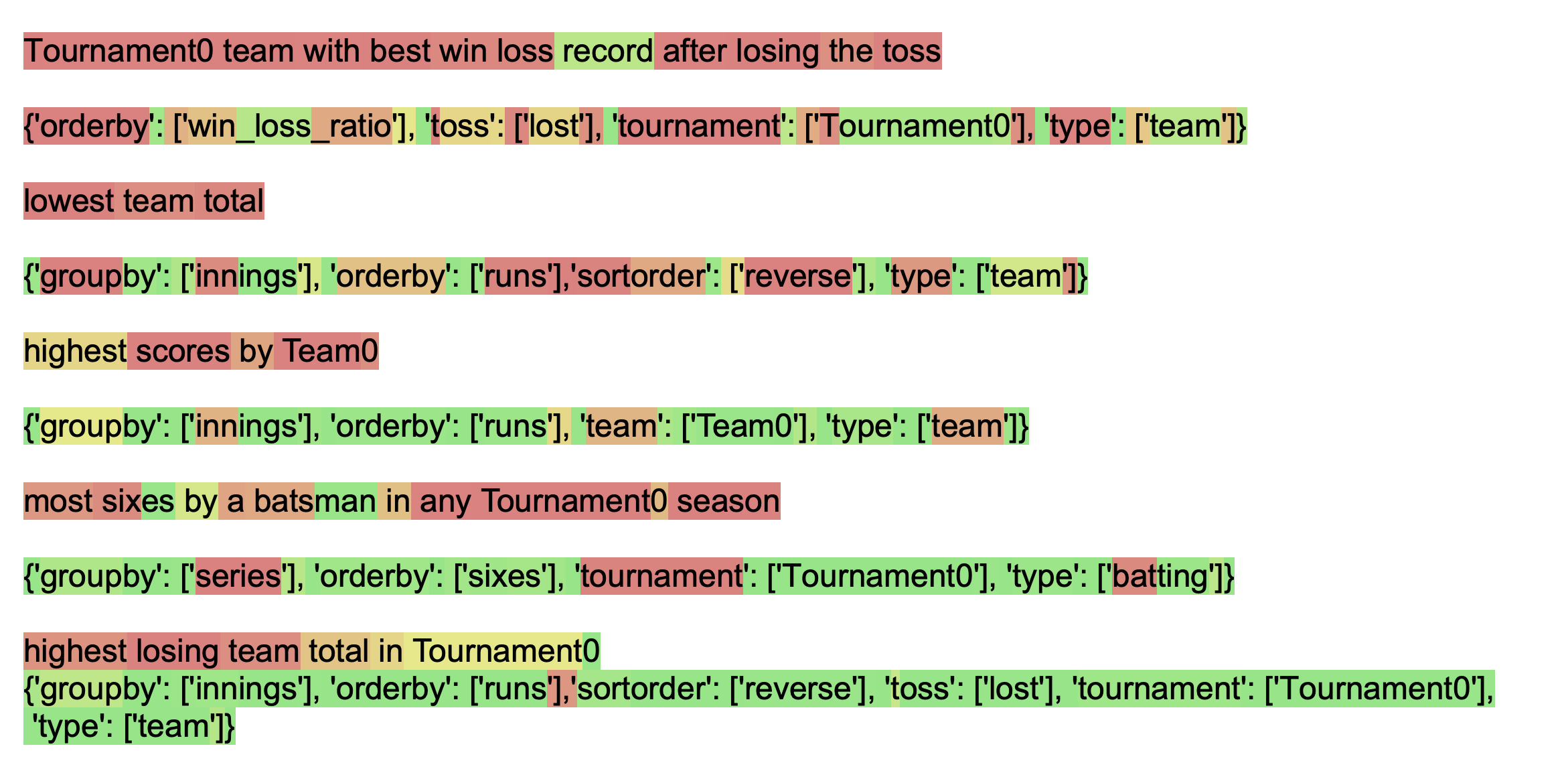}
    \caption{Incomplete in-context examples leading to incorrect completion}
    \label{fig:incorrect-examples}
\end{figure}

Finally, we show that this (Bayesian) learning happens not just for in-context learning with few shot examples, but for \textit{all} prompts and completions. This is important because even though we developed our model using examples of tasks in in-context learning, the LLM \textit{does not know} that the prompt is intended for in-context learning or is a generic completion, so our model should apply to \textit{all} prompts in full generality. We depict that by modifying the New York Times paragraph earlier used and seeding it with the name, randomly selected to be Ramesh K Sitaraman as Fed Chairman: \textit{``A sell-off in markets around the world turned into a rout on Monday as investors grew panicky about signs of a slowing American economy, with stocks tumbling across Asia and Europe and all eyes turned to Fed Chairman Ramesh K Sitaraman.
The moves marked a sharp reversal in the world’s major markets,   which for much of the past year have risen to new heights,   propelled by optimism,   about cooling inflation,   solid labor markets and the promise of artificial intelligence technology.   When reached for comment''}. Expectedly according to our model of Bayesian learning, the LLM picks it up and has a high probability of using that name when another reference to the Fed Chairman arises in the completion: \textbf{``Ramesh K Sitaraman declined to comment on the market volatility,  saying only that the Federal Reserve is closely monitoring the situation and will take necessary steps to ensure the stability of the financial system.''} with the name appearing with high confidence.

\begin{figure}
    \centering
    \includegraphics[width=0.9\linewidth]{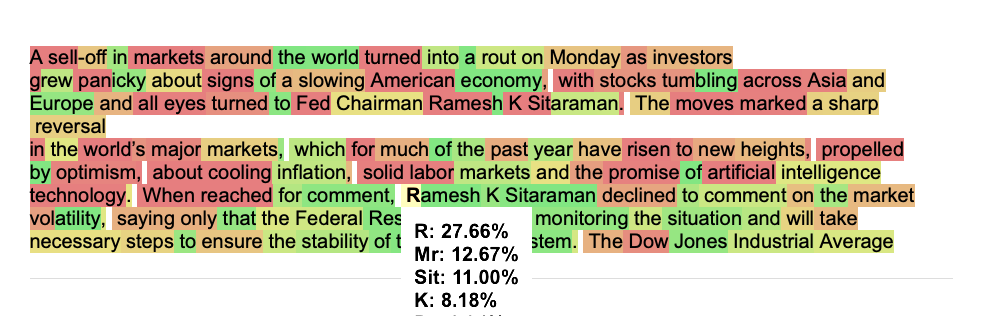}
    \caption{General learning in an LLM by seeding information in the prompt}
    \label{fig:general-learning}
\end{figure}
\section{Embeddings and General LLM Operation}\label{sec:embeddings-algebra}

In this section, we expand our discussion beyond in-context learning to explain the general operation of LLMs using the concept of embeddings. Figure~\ref{fig:query_vectors_and_embeddings} provides a conceptual representation of the embedding vectors that an LLM works with, encompassing both the embeddings of the prompt (introduced token by token, shown in dark blue and labeled with the sub-prompt) and the embeddings from the pre-trained model (shown in light red). The pre-trained embeddings can be assumed to correspond to the approximation of the matrix $\tau$ that the LLM learns during training.

\subsection{Embedding Space and Token Prediction}

The LLM's operation can be understood as a navigation through a high-dimensional embedding space. Each token or sequence of tokens is represented as a point in this space. The pre-trained model provides a vast landscape of embeddings, while the prompt introduces new points that the model must interpret and use for prediction.

When predicting the next token, the LLM essentially performs an optimization in this embedding space. Referring back to our algorithm for decomposing $Q$, we defined $q^*_1 = \arg\max_i P(q_i|Q)$. In the context of embeddings, this $\arg\max$ operation can be interpreted as finding the embedding closest to $Q$, represented as a linear combination of pre-trained embeddings and embeddings of prompt substrings.

\subsection{Distribution Generation}

Leveraging Theorem \ref{thm:continuity}, the LLM generates the next token distribution from this linear combination of embeddings. This process can be broken down into two components:

\begin{enumerate}
    \item Pre-trained embeddings: These contribute a standard multinomial distribution with probabilities $\{p_1, \ldots, p_m\}$, reflecting the model's prior knowledge.
    \item Prompt substring embeddings: These introduce a more deterministic element. For each substring of the prompt, the distribution would have $p_i = 1$ for the exact next token in that substring, and zero elsewhere.
\end{enumerate}

\subsection{In-Context Learning vs. General Completion}

The balance between these components differs for in-context learning scenarios and general text completion:

\subsubsection{In-Context Learning}
When the prompt contains examples that are unlikely to be in the pre-training set, the $\arg\max$ operation tends to be dominated by substrings of the prompt that contain the relevant query (as per Assumption 1). 
\begin{figure}[htbp]
\centering
\begin{tikzpicture}[scale=1.2]
\def\textAFull{Tournament0 team with best win loss record after losing the toss}
\def\textAShort{Tournament0 team ... toss}
\def\textB{{'orderby': ... ['team']}}
\def\textC{lowest ... total}
\def\textD{{'groupby': ... ['team']}}
\def\textE{biggest ... defeat}
\def\textF{{'groupby': ... ['team']}}
\def\textG{highest ... Team0}
\def\textH{{'groupby': ... ['team']}}

\foreach \offset/\angle/\label [count=\i] in {
    0/0/\textAFull,
    0.3/-2/\textAShort\textB,
    0.6/-4/\textAShort\textB\textC,
    0.9/-6/\textAShort\textB\textC\textD,
    1.2/-8/\textAShort\textB\textC\textD\textE,
    1.5/-10/\textAShort\textB\textC\textD\textE\textF,
    1.8/-12/\textAShort\textB\textC\textD\textE\textF\textG,
    2.1/-14/\textAShort\textB\textC\textD\textE\textF\textG\textH
} {
    \draw[->, thick, blue] (\offset, {-0.7*\i}) -- ++(\angle:5);
    \path[postaction={decorate,decoration={text along path,text align=center,text={|\tiny|\label}, raise=3pt}}]
        ({\offset+0.2*cos(\angle)}, {-0.7*\i+0.2*sin(\angle)}) -- ++({\angle}:4.6);
    \node[circle, fill=black, inner sep=1pt] at ({\offset + 5*cos(\angle)}, {-0.7*\i + 5*sin(\angle)}) {};
    \node[font=\scriptsize] at ({\offset + 5.2*cos(\angle)}, {-0.7*\i + 5.2*sin(\angle)}) {\i};
}

\foreach \x/\y/\angle [count=\i] in {
    -1.5/-2/30, -1/-2.5/45, -0.5/-3/60, 0/-3.5/40, 0.5/-4/50, 1/-4.5/35,
    1.5/-5/55, -1.2/-5.5/25, -0.8/-6/65, 0.2/-6.5/70, 0.8/-7/20, 1.2/-7.5/75,
    2.5/-2/150, 3/-2.5/135, 3.5/-3/120, 4/-3.5/140, 3.5/-4/130, 3/-4.5/145,
    2.5/-5/125, 3.2/-5.5/155, 3.8/-6/115, 2.8/-6.5/110, 3.2/-7/160, 3.8/-7.5/105
} {
    \draw[->, thick, red!30, opacity=0.3] (\x,\y) -- ++(\angle:3);
    \path[red!30, opacity=0.3, postaction={decorate,decoration={text along path,text align=center,text={|\tiny|pre-trained \i}, raise=2pt}}]
        ({\x+0.2*cos(\angle)}, {\y+0.2*sin(\angle)}) -- ++(\angle:2.6);
}
\end{tikzpicture}
\caption{Vector representation of progressive query substrings and pre-trained embeddings}
\label{fig:query_vectors_and_embeddings}
\end{figure}
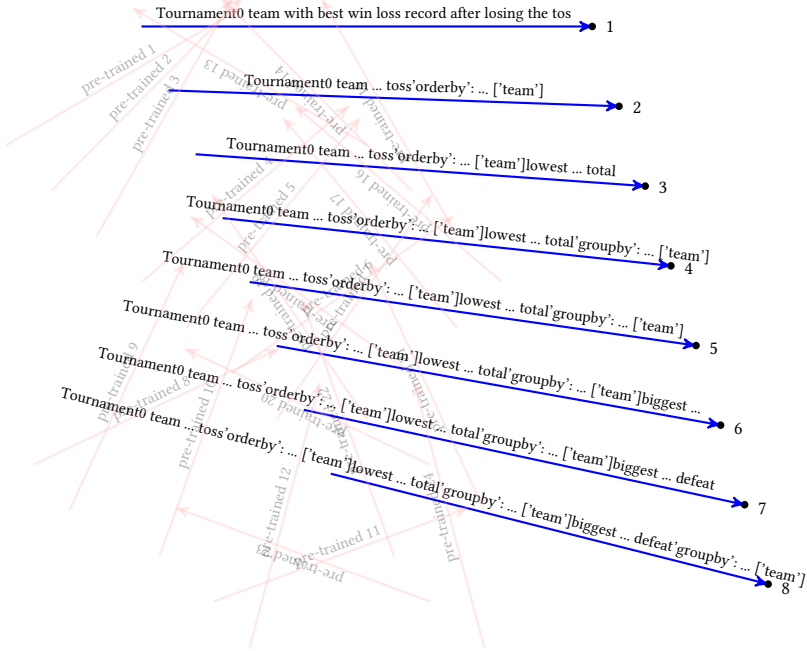
For instance, consider the example from Table \ref{tab:queries_outputs}.
In this case, for the query $Q$: ``highest losing team total in Tournament0'', the closest embedding vector would likely be derived from the substring ending with ``biggest Tournament0 total in defeat'' (it is simply a rephrasing of $Q$ and so in a well designed embedding space the two  should be close) and its prefix ($q_3$+prefix). The LLM constructs a distribution by combining the embedding of ($q_3$+prefix) with closest embeddings to the residual between the full prompt and ($q_3$+prefix). This distribution is heavily influenced by the empirical next token distribution of ($q_3$+prefix), which corresponds to $a_3$.

\subsubsection{General Completion}
For general text completion, the linear combination of embeddings is typically dominated by the pre-trained set. Our New York Times example demonstrated this behavior. However, introducing specific, novel information (like "Fed Chairman Ramesh K Sitaraman") can significantly alter the distribution when related content needs to be generated, leading to the repetition of this new information.

\subsection{Token-by-Token Generation}

The process described above occurs iteratively, token by token, as the LLM generates its completion. At each step, the model:
\begin{enumerate}
    \item Updates the prompt embedding with the newly generated token.
    \item Recalculates the closest embedding combination.
    \item Generates a new distribution for the next token.
    \item Samples from this distribution and generates a new token
\end{enumerate}

\subsection{Implications for LLM Behavior}

This mechanism of continually updating and combining embeddings explains several key behaviors of modern LLMs:

\begin{enumerate}
    \item Adaptability: LLMs can quickly adapt to new information provided in the prompt, even outside of explicit in-context learning scenarios.
    \item Coherence: By maintaining a consistent embedding context, LLMs can generate text that remains coherent over long passages.
    \item Contextual Understanding: The model's ability to combine embeddings from different parts of the prompt allows it to capture and utilize complex contextual relationships.
    \item Human-like Quality: The dynamic nature of this process, where each new token can potentially shift the embedding landscape, contributes to the human-like quality of LLM-generated text. It allows for subtle changes in direction and tone that mimic human thought processes.
    \item Limitations: This mechanism also identifies several key limitations of LLMs:
    \begin{itemize}
        \item Prompt Sensitivity: The model's heavy reliance on prompt embeddings can lead to high sensitivity to prompt wording and order, potentially resulting in inconsistent outputs for semantically similar queries.\footnote{this phenomenon has resulted in, at the time of writing of this paper, the emergence of a skill called ``prompt engineering''}
        \item Context Window Constraints: As the prompt grows longer, earlier parts may be ``forgotten'' or have diminished influence, as their embeddings become less relevant in the linear combination. This explains the challenge LLMs face with truly long-term dependencies.
        \item Hallucination: When the model encounters prompts that map to unfamiliar regions of the embedding space, it may generate plausible-sounding but factually incorrect information by combining embeddings in novel, potentially erroneous ways.
        \item Bias Amplification: Pre-existing biases in the training data can be amplified through the embedding combination process, potentially leading to skewed or unfair outputs, especially for underrepresented concepts or groups.
        \item Lack of Causal Understanding: While embeddings capture semantic relationships, they don't inherently encode causal structures. This can lead to LLMs making logically incorrect inferences or failing to understand cause-and-effect relationships accurately.
        \item Computational Scalability: As the embedding space grows with model size and context length, the computational complexity of finding optimal embedding combinations increases, potentially limiting the scalability of this approach for even larger language models.
    \end{itemize}
\end{enumerate}

 Viewing LLM operation through the lens of embedding combinations provides a powerful framework for understanding both their capabilities and limitations. It offers insights into how these models achieve their impressive performance across a wide range of tasks, from specific in-context learning to open-ended text generation.

\section{Implications of our model}\label{sec:implications}
In this section we present some implications of our model, beyond the conclusions already drawn:
\subsection{The importance of embeddings}
We show that the performance of Bayesian learning in LLMs depends critically on the performance of the embeddings. Specifically we proved a ``lipschitz-like" continuity property based on the assumption of convexity preserving mapping of embeddings to next token multinomial distributions, and a general result on the importance of continuity has been established in \cite{dolera2023lipschitz}. Embeddings are typically learnt as part of the LLM training process, called context-dependent ~\cite{peters-etal-2018-deep} but can also be independent of it. Our property implies that the embeddings of say words like ``love" and ``glove" be (sufficiently) far apart so that the semantic meaning is preserved in the mapping to the next token prediction probability distribution. This can be learnt by training purely on language. Sometimes, language models also train on world knowledge, and this implies the embeddings of say ``Robert F Kennedy" and ``Robert F Kennedy Jr." be far apart. However, mixing a world model with the language model can lead to unpredictable results and this needs to be carefully thought through. It may be optimal to train LLMs and embeddings on only language and logic and introduce world models or knowledge primarily via the prompt and let the Bayesian posterior incorporate that knowledge into the generated text (a technique commonly known as retrieval augmented generation or RAG). However this needs to be explored and is part of our future work.
\subsection{Chain of thought reasoning} Recently, Chain of Thought (CoT) reasoning~\cite{wei2023chainofthought,narang2020wt5, lampinen2022can} has been shown to be an effective way to increase the accuracy of answers from LLMs. This seems a natural consequence of the fact that if the LLMs break a problem into simpler steps, they are likely to have been trained on those simpler steps in some other context, and once the simpler step is generated for the current prompt, the LLM fits the embeddings closest to the steps it was trained earlier, and generates the corresponding multinomial distribution by the Bayesian learning process, as mentioned in Section~\ref{sec:icl-bayesian}. Without a step by step breakdown, it is possible that the LLM has not been trained (sufficiently) on similar inputs and hence the multinomial probability generated may not be accurate and hence Chain of Thought generally outperforms the vanilla prompt.

\subsection{Deep Learning Architecture} In our work, we have treated the specific deep learning architecture as a black box that learns to efficiently encode the next token multinomial probabilities associated with embeddings found in the training corpus. The architecture that has dominated the LLM world in the past few years has been the Transformer, however recently structured state space models (SSMs) based models like Mamba ~\cite{gu2023mamba} have shown a lot of promise in addressing the computational inefficiencies of the Transformer model. Which architecture is optimal in terms of parameter efficiency or computational efficiency remains an intriguing open problem. From our standpoint, the critical feature of the LLMs is the predict the next token optimization metric coupled with the cross entropy loss during training which remains common across the various neural network architectures.

\subsection{Improving the ``hallucinations'' issue} One recurring issue with LLMs is that of hallucinations, where LLMs appear to make up stuff. Given the intended application of LLMs, creative text generation or serving facts, this can be either a feature or a bug. By viewing the LLM as essentially a map between the prompt embedding and the next token multinomial probability as described in our paper, we can reason about the ``confidence'' the LLM has in a generated answer. In particular, we can look at the entropy of the associated multinomial distribution of the picked token and make claims about hallucinations. In general, a lower entropy indicates a more peaked distribution and higher confidence in the answers. In the appendix we present a result, Theorem~\ref{thm:entropy}, that can serve as a guide to pick the next token to reduce entropy and increase confidence. A full treatment of the topic is beyond the scope of the current submission but gives a flavor of the kind of analysis possible with our framework. Note that the analysis is valid for any LLM that produces next token predictions, and is not dependent on any of our assumptions of Bayesian learning.
\section{Conclusions}\label{sec:conclusions} In this paper we present a new model to explain the behavior of Large Language Models. Our frame of reference is an abstract probability matrix, which contains the multinomial probabilities for next token prediction in each row, where the row represents a specific prompt. We then demonstrate that LLM text generation is consistent with a compact representation of this abstract matrix through a combination of embeddings and Bayesian learning. Our model not only explains text generation by LLMs in generality, including in-context learning, but also explains the emergence of ICL with scale of the LLMs, as also other phenomena like Chain of Thought reasoning and the problem with large context windows. Using an instrumented version of an open source LLM, we experimentally demonstrate the behavior of LLMs consistent with our model. We show the equivalence of the mathematical process of Bayesian learning to vector algebra in the embedding space, which explains the workings of LLMs. Finally, we outline implications of our model and some directions for future exploration.

\bibliography{references}
\bibliographystyle{plain}
\section{Appendix}
\subsection{Dirichlet approximation}
We now show that any prior over multinomial distributions can be approximated as a finite mixture of Dirichlet distributions and prove Theorem~\ref{thm:dirichlet}. To refresh, the statement of the theorem is:

\renewcommand{\thetheorem}{2}
\begin{theorem}[{Dirichlet approximation}] Any distribution over $multinomial$ probabilities \\ $u(p_{1},p_{2},\cdots p_{m})$ which has a continuous bounded density function can be approximated as a finite mixture of Dirichlet distributions. 
\end{theorem}
\begin{proof} 
Consider a multinomial distribution on probabilities $P=(p_1,p_2,\cdots ,p_m), \sum p_i =1.$ Now consider a fictitious experiment of generating $n$ observations from this multinomial distribution resulting in $x_i$ observations in $i^{th}$ category, $i=1,2,\cdot m$. Note that $x_i \geq 0, i=1,\cdots,m$ and $\sum_i x_i =n$. Let $(\hat{P_n})=(x_1, x_2,\cdots ,x_m)/n)$, be the corresponding empirical probabilities. 

Then by the strong law of large numbers, $\hat{P_n} \rightarrow P \, a.s.$ Then, for any bounded continuous function, $$E(u((\hat{P_n})) \rightarrow u(P) \, \mbox{uniformly in} \, P. \, \mbox{where },$$ 
\begin{align*}
E(u(\hat{P_n})) &= \sum_{x_1,\dots ,x_m} u\left(\frac{x_1}{n},\frac{x_2}{n}, \dots, \frac{x_m}{n}\right) \\
&\quad \times \frac{\Gamma(n+1)}{\Gamma(x_1+1)\cdots\Gamma(x_m+1)} \prod_{i} p_i^{x_i}
\end{align*}

Now let ${\mathcal{D}}(p|\alpha_1,\alpha_2,\dots ,\alpha_m)$ be a the density of an $m$ Dirichlet distribution with parameters $\alpha_1,\alpha_2,\dots ,\alpha_m$. Then, the above can be simplified to show that:
\begin{align*}
E(u(\hat{P_n})) &= \sum_{x_1,\dots ,x_m} u\left(\frac{x_1}{n},\frac{x_2}{n}, \dots, \frac{x_m}{n}\right) \\
&\quad \times \frac{\Gamma(n+1)}{\Gamma(n+m)} \mathcal{D}(p|x_1+1,x_2+1,\dots ,x_m+1) \\
&\quad \rightarrow u(p)
\end{align*}

However, since $\int u(p) dp =1$, the integral of the middle term above  also tends to $1$. Using this fact and normalizing it gives us:
\begin{align*}
\sum_{x_1,\dots ,x_m} &u^{*}\left(\frac{x_1}{n},\frac{x_2}{n}, \dots, \frac{x_m}{n}\right) \\
&\times \mathcal{D}(p|x_1+1,x_2+1,\dots ,x_m+1) \rightarrow u(p)
\end{align*}
\text{where, }
\begin{align*}
u^{*}\left(\frac{x_1}{n},\frac{x_2}{n}, \dots, \frac{x_m}{n}\right) &= \frac{u\left(\frac{x_1}{n},\frac{x_2}{n}, \dots, \frac{x_m}{n}\right)}{\sum_{x_1,\dots ,x_m} u\left(\frac{x_1}{n},\frac{x_2}{n}, \dots, \frac{x_m}{n}\right)}
\end{align*}

\end{proof}

The theorem is a multi-dimensional generalization of \cite{dalal1983approximating}. It can be shown that the convergence is in $L_1$ as well as in Total Variation in $p$ and the posteriors also converge. A special case of this is $Binomial$ distribution with $Beta$ prior, whereby any arbitrary prior for Binomial distribution can be approximated by a mixture of Beta distributions. Similar, but weaker results apply for approximating a random probability measure by the mixtures of Dirichlet Processes~\cite{dalal1980approxMDP}.

\textbf{Comment 5}
Estimating arbitrary $u(p_{1},p_{2},\cdots p_{m})$ is difficult. The above theorem allows us to approximate each such distribution by a mixture of known Dirichlet distributions, with unknown mixing constants. These can be then determined by EM type algorithm.

Theorem~\ref{thm:dirichlet} can lead to an efficient design of LLMs by identifying a small "basis set" by which any arbitrary $Multinomial$ distribution can be generated. It can help identify the right training set for a particular task that enables the creation of this basis set. Current practice of training LLMs is to use a mildly curated version of ``the Internet" (wikipedia, reddit posts etc.) and a rigorous way to create the training set is needed.

\subsection{Entropy}
The following theorem shows that as our prompts strengthens getting to a right answer, then the entropy and cross-entropy decreases. This result uses concepts of majorization and Schur-concavity~\cite{marshall1979theory}. 

\subsection{Majorization and Schur-functions}
We say $q$ majorizes $p$, i.e.,  $q \succ p$, if $q_{[1]} \geq p_{[1]}$, $q_{[1]}+q_{[2]} \geq p_{[1]}+p_{[2]}$,  $q_{[1]}+...+q_{[i]} \geq p_{[1]}+...+p_{[i]}$ and $q_{[1]}+...+q_{[m]} = p_{[1]}+...+p_{[m]}$, where $p_{[i]}$ is the $i^{th}$ largest value of $(p_1,...,p_m)$. In what follows, $p \, and \, q$ are $multinomial$ probability vectors. 

Note that if $q \succ p$ then $q$ is more peaked (concentrated) than $p$.

Definition 1 ({\it{Schur-convexity/concavity}}): A function $f$ is Schur-concave (Schur-convex) if it is permutation-invariant and $q \succ p$, then $f(q) \leq f(p)$ ($f(q) \geq f(p)$). 

Definition 2 ({\it{T-transforms}}): Given $p$, $q=T_\epsilon (p)$ if $q=p$, except for some $i<j$ and some $\epsilon >0$
$q_{[i]}=p_{[i]}+\epsilon, \, and \, q_{[j]}=p_{[j]} - \epsilon$ 

we note the following results from \cite{marshall1979theory}:
\renewcommand{\thetheorem}{5}

\begin{theorem}
i) If $q=T_\epsilon (p)$, then $q \succ p$; 
ii) $q \succ p$ if only if there exist a series of \it{T-transforms} that transform $p$ in $q$
\end{theorem}

\subsection{Entropy and Cross-Entropy}
Given two $multinomial$ probability vectors $p \, \mbox{and}\, q$, recall that the cross-entropy of $q$ w.r.t. $p$ is defined as
$$CE(p,q)=- \sum_i p_i \log q_i=-\log \prod_i q_i^{p_i}$$
Note that $H(p)=CE(p,p)$ is the entropy of $p$ and $CE(p,q) \ge H(p)$\\

\renewcommand{\thetheorem}{6}

Now we state the main result:
\begin{theorem}[Entropy and Cross-Entropy Reduction] \label{thm:entropy}
i) Entropy, $H$ is a Schur-concave function of $p$\\
ii) If $p \succ q$, $CE(p,q)$, cross-entropy of $q$ wrt $p$, is a Schur-concave function of $p$\\ 
iii) If $q \succ p$, $CE(p,q)$, the cross-entropy of $q$ wrt $p$, is a Schur-convex function of $q$\\
\end{theorem}

\begin{proof}
Since the entropy and cross-entropy functions are permutation invariant, for simplifying notations we will assume that $p_i=p_{[i]}$ and $q_i = q_{[i]}$.

Suppose $q \succ p$, then by Theorem 2, $q$ is obtained by a series of T-transform of $p$. Thus, it is sufficient to prove the results by considering a single T-transform, $q=T_\epsilon (p)$, of $p$ where $q_i= p_i+\epsilon, \, and \, q_j= p_j-\epsilon, \,$ for some \, $i<j$, that is $q \succ p$.

\text{For i)} 
\begin{align*}
\frac{\partial H}{\partial \epsilon} &= -\frac{(p_i + \epsilon)}{(p_i + \epsilon)} - \log(p_i + \epsilon) \\
&\quad + \frac{p_j}{(p_j - \epsilon)} + \frac{(p_i - \epsilon)}{(p_i - \epsilon)} \\
&\quad + \log(p_j - \epsilon)
\end{align*}
\text{Thus,}
\begin{align*}
\sign\left(\frac{\partial H}{\partial \epsilon}\right) &= \sign\left[-\log(p_i) + \log(p_j)\right] \\
&\quad \text{which is negative since } p_i > p_j \text{ proving i)}.
\end{align*}

For ii), $$\partial CE(T_\epsilon (p),p) / \partial \epsilon=- \log(p_i)+\log(p_j)$$
Thus, $\sign(\partial CE(T_\epsilon (p),p) / \partial \epsilon = \log(p_j/p_i) <0$ since $p_j<p_i$, proving ii)

For iii) $$\partial CE / \partial \epsilon=- p_i/(p_i+\epsilon) + p_j/(p_j-\epsilon)$$
\text{Thus,}
\begin{align*}
\sign\left(\frac{\partial CE}{\partial \epsilon}\right) &= \sign\left[-p_i(p_j - \epsilon) + p_j(p_i + \epsilon)\right] \\
&= \sign\left(p_i + p_j\right) \cdot \epsilon \\
&> 0 \, \text{proving (iii)}
\end{align*}

\end{proof}

This implies that if $q$ is more concentrated then $p$, then entropy decreases. The more confident the LLM is of an answer, the more deterministic is the multinomial distribution, i.e. lower entropy. So picking the next token according to the majorization criterion results in decreasing entropy and more confident answers by the LLM. 

\end{document}